\documentclass[11pt]{article}
\usepackage{fullpage}
\usepackage{graphicx} %
\usepackage{xcolor}
\usepackage{amsmath}
\usepackage{amsfonts}
\usepackage{amsthm}
\usepackage{natbib}
\usepackage{hyperref}
\usepackage{cleveref}

\usepackage{color-edits}

\newcommand{\sr}{\mathsf{SR}}

\newcommand{\I}{\mathbb I}
\newcommand{\E}{\mathbb E}
\newcommand{\R}{\mathbb R}
\newcommand{\bbP}{\mathbb P}
\newcommand{\bbL}{\mathbb L}
\newcommand{\cP}{\mathcal P}
\newcommand{\cL}{\mathcal L}
\newcommand{\subopt}{\mathsf{subopt}}
\newcommand{\CAL}{\mathsf{CAL}}

\newcommand{\poly}{\mathsf{poly}}
\newcommand{\sps}[1]{^{(#1)}}

\DeclareMathOperator*{\argmin}{arg\,min}

\renewcommand{\tilde}{\widetilde}

\newtheorem{theorem}{Theorem}[section]
\newtheorem{lemma}[theorem]{Lemma}

\newtheorem{corollary}[theorem]{Corollary}
\newtheorem{definition}[theorem]{Definition}
\newtheorem{problem}[theorem]{Problem}

\title{Near-optimal Swap Regret Minimization for Convex Losses}
\author{Lunjia Hu \\ Northeastern University \\ \texttt{lunjia@alumni.stanford.edu} \and Jon Schneider \\ Google \\ \texttt{jschnei@google.com} \and Yifan Wu \\ Microsoft Research \\ \texttt{yifan.wu2357@gmail.com}}
\date{}

\allowdisplaybreaks
\begin{document}

\maketitle
\begin{abstract}
We give a randomized online algorithm that guarantees near-optimal $\widetilde O(\sqrt T)$ expected swap regret against any sequence of $T$ adaptively chosen Lipschitz convex losses on the unit interval. This improves the previous best bound of $\widetilde O(T^{2/3})$ and answers an open question of \citet{full-swap}. In addition, our algorithm is efficient: it runs in $\poly(T)$ time. A key technical idea we develop to obtain this result is to discretize the unit interval into bins at multiple scales of granularity and simultaneously use all scales to make randomized predictions, which we call multi-scale binning and may be of independent interest. A direct corollary of our result is an efficient online algorithm for minimizing the calibration error for general elicitable properties. This result does not require the Lipschitzness assumption of the identification function needed in prior work, making it applicable to median calibration, for which we achieve the first $\widetilde O(\sqrt T)$ calibration error guarantee.

\end{abstract}

\section{Introduction}

Swap regret minimization is a fundamental problem in the field of online learning, with a recent uptick in interest due to its many applications to calibration \citep{foster1997calibrated, foster1998asymptotic}, fairness \citep{hebert2018multicalibration, gopalan2023swap}, and learning in games \citep{deng2019strategizing}. Like the standard notion of external regret, swap regret measures the suboptimality of a sequence of actions taken by an online learner against an adversary; however, whereas minimizing external regret only guarantees that there is no single action that performs better in hindsight, minimizing swap regret provides the much stronger guarantee that there is no consistent transformation of the learner's actions (e.g., substituting action $A$ every time the learner played action $B$) that outperforms the learner. 

To this date, the majority of results on swap regret minimization focus on settings where the losses incurred by the learner are \emph{linear} functions of the learner's actions (for example, in the setting originally studied by \citet{blum-mansour}, every round the learner picks a mixed strategy $p \in \Delta_{n}$ over $n$ actions, the adversary picks an $n$-dimensional loss $\ell \in [-1, 1]^n$, and the learner incurs the linear loss $\langle p, \ell\rangle$). Recently, \citet{full-swap} initiated the study of swap regret minimization against sequences of \emph{non-linear} convex losses, showing that many important problems ($L_2$ calibration, minimization of swap regret in structured games) can be conveniently framed in this language.

Formally, \citet{full-swap} introduce the following problem. A Learner interacts with an Adversary for $T$ rounds. In each round $t \in [T]$, the Learner begins by choosing a distribution $\rho_t$ supported on $[0, 1]$. The Adversary then reveals a $1$-Lipschitz convex loss function $\ell_t:[0, 1] \to \R$. The Learner then samples an action $p_t$ from $\rho_t$ and incurs loss $\ell_t(p_t)$. The goal of the Learner is to minimize their expected swap regret $\E[\sr(p_{1:T};\ell_{1:T})]$, where

\begin{equation}
\label{eq:sr}
\sr(p_{1:T};\ell_{1:T}):= \sup_{\sigma:[0,1]\to [0,1]} \left(\sum_{t = 1}^T \ell_t(p_t) - \sum_{t = 1}^T\ell_t(\sigma(p_t))\right).
\end{equation}

If all the losses $\ell_t$ are linear functions in $p$, then one can attain $\tilde{O}(\sqrt{T})$ swap regret by simply running a standard swap-regret minimization algorithm over the two endpoints of the Learner's action set $[0, 1]$ (in particular, any intermediate action $p \in (0, 1)$ can be replaced by a distribution playing $1$ with probability $p$ and $0$ with probability $1-p$ without changing the Learner's expected loss). Similarly, if the losses $\ell_t$ are strongly convex, \citet{full-swap} show that it is possible to obtain $\tilde{O}(T^{1/3})$ (pseudo\footnote{\citet{full-swap} study a slightly weaker definition of swap-regret than defined above, where the expectation over sampling from $\rho_t$ takes place within the supremum of \eqref{eq:sr} (i.e., for any specific swap function $\sigma$ they guarantee $\tilde{O}(T^{1/3})$ expected swap-regret). This was later strengthened by \cite{luo2025simultaneous}, who proved an $\tilde{O}(T^{1/3})$ expected swap-regret bound for sequences of smooth convex losses arising from online calibration instances.}-)swap regret bounds, by taking advantage of $O(\log T)$ external regret bounds for such losses. 

Interestingly, despite the $\tilde{O}(\sqrt{T})$ swap regret bound for linear losses and the $\tilde{O}(T^{1/3})$ swap regret bound for smooth, strongly convex losses, the best known attainable for arbitrary (bounded and Lipschitz) convex losses is $\tilde{O}(T^{2/3})$. This is worse than either of the two above bounds, and follows from simply discretizing the Learner's action set and running a classical swap-regret algorithm over the discretization. This is particularly counter-intuitive because minimizing external regret for convex losses is \emph{no harder} than minimizing external regret for linear losses -- it suffices to simply minimize external regret with respect to the linearizations of loss function formed by the subgradients of $\ell_t$ at the points $p_t$ played by the Learner. Unfortunately, this reduction does not extend to swap regret minimization\footnote{The subtlety here is that the linearization of $\ell_t$ depends on the specific action $p_t$ sampled from $\rho_t$. The reduction for external regret minimization sidesteps this because the Learner can always play a deterministic strategy (i.e., a singleton $\rho_t$) -- in contrast, there are provably no sublinear regret deterministic algorithms for swap regret minimization.}. \citet{full-swap} pose as an open question whether it is possible to improve this bound to match the $\tilde{O}(\sqrt{T})$ bound for linear losses.

\subsection{Our Contributions}

We positively resolve this open question, providing an efficient algorithm achieving $\tilde{O}(\sqrt{T})$ swap regret against an arbitrary adaptive sequence of (bounded and Lipschitz) convex losses. 

\begin{problem}[Swap regret minimization for convex losses]
\label[problem]{def:main}
In each round $t = 1,\ldots, T$, 
\begin{enumerate}
    \item predictor (i.e.\ learner) chooses a distribution $\rho_t$ on $[0,1]$;
    \item adversary reveals a $1$-Lipschitz convex loss function $\ell_t:[0,1]\to \R$;
    \item prediction $p_t\in [0,1]$ is drawn from $\rho_t$, and predictor incurs loss $\ell_t(p_t)$.
\end{enumerate}
The predictor's goal is to minimize the expected swap regret $\E[\sr(p_{1:T};\ell_{1:T})]$ defined in \eqref{eq:sr}.
\end{problem}

\begin{theorem}
\label{thm:main}
    For every positive integer $T \ge 2$, there exists a prediction strategy for \Cref{def:main} that guarantees
    \[
    \E[\sr(p_{1:T};\ell_{1:T})] = O(\sqrt T\log T).
    \]
Moreover, there exists such a prediction strategy that runs in $\poly(T)$ time.
\end{theorem}
\Cref{thm:main} is optimal up to the $O(\log T)$ factor because there is a simple way to choose the losses $\ell_{1:T}$ to incur $\Omega(\sqrt T)$ expected swap regret against any prediction strategy (\Cref{lem: lowerbound median}). In addition to the expectation bound in \Cref{thm:main}, our prediction algorithm also guarantees a high-probability bound: given any $\delta\in (0,1/T]$, the algorithm achieves $O\left(\sqrt{(T\log T)\log(1/\delta)}\right)$ swap regret with probability at least $1-\delta$ (\Cref{lm:efficient}).

\subsubsection{Implications to Calibration for Elicitable Properties}
\label{subsec: elicitable property}

As a motivating application of our results, we describe how Theorem~\ref{thm:main} leads to improved algorithms for the problem of \emph{calibrated forecasting of elicitable properties}. In particular, we show that such calibration measures of common properties (such as mean, median, and $q$-quantile) can often be written as swap regrets of scoring functions that are Lipschitz and convex, allowing Theorem~\ref{thm:main} to recover the first $\tilde{O}(T^{1/2})$ calibration bounds for such properties. 

\paragraph{Proper Scoring Rules and Elicitable Properties.} 
Many prediction tasks are to estimate a \emph{statistical property} of an unknown outcome distribution.
Formally, let $\tau$ be a distribution over outcomes $Y$, and let $\Gamma$ be a (possibly set-valued) property mapping
that assigns to each $\tau$ a subset $\Gamma(\tau)\subseteq [0,1]$ of optimal reports (e.g., probabilities, quantiles, or other statistics).
A predictor outputs a report $p\in[0,1]$, after which an outcome $y\sim \tau$ is realized.

A \emph{scoring function} (or loss) is a function $S:[0,1]\times Y\to\R$ that evaluates the report $p$ when the realized outcome is $y$. 
The score $S$ is \emph{proper}, if it \emph{elicits} the property $\Gamma$: for every distribution $\tau$ over $Y$, the set of Bayes-optimal reports under $\tau$
coincides with $\Gamma(\tau)$:
\begin{equation}
\label{eq:elicit}
\argmin_{p\in[0,1]} \E_{y\sim \tau}\!\left[S(p,y)\right] \;=\; \Gamma(\tau).
\end{equation}
A property $\Gamma$ is called \emph{elicitable} if there exists some scoring function $S$ that elicits it.
Assuming the outcome space is $Y = [0, 1]$, the following scoring functions elicit the mean,
the median, and the $q$-quantile:
\begin{align}
    &S_{\mathsf{mean}}(p, y) = (p-y)^2,\\
    &S_{\mathsf{median}}(p, y) = |p-y|,\\
    &S_{q\text{-}\mathsf{quantile}}(p,y) = \I\left[y\geq p\right]q\cdot(p-y) + \I\left[y\leq p\right](1-q)\cdot(y-p)\qquad 
\end{align}
Note that all these scoring functions $S(\cdot, y)$ are convex and $O(1)$-Lipschitz for any fixed $y$. 

\paragraph{Calibration Error for Elicitable Properties.} Given a scoring function $S:[0,1]\times Y \to \R$, we define the sub-optimality of $p\in [0,1]$ w.r.t.\ distribution $\tau$ on $Y$ as follows:
\[
\subopt_\tau(p) := \E_{y\sim \tau}S(p,y) - \inf_{p^*\in [0,1]}\E_{y\sim \tau}S(p^*,y) \ge 0.
\]
When $\Gamma$ is the property elicited by $S$ as in \eqref{eq:elicit}, we have $\subopt_\tau(p) = 0$ if and only if $p\in \Gamma(\tau)$.
For sequences $p_1,\ldots,p_T\in [0,1]$ and $y_1,\ldots,y_T\in Y$, we define the calibration error $\CAL_S$ as follows:
\begin{equation}
\label{eq:cal-subopt}
\CAL_S(p_{1:T};y_{1:T}) := \sum_{p\in P}n_p\cdot \subopt_{\tau_p}(p),
\end{equation}
where $P = \{p_1,\ldots,p_T\}$ is the set of distinct predictions, $n_p = |\{t\in [T]:p_t = p\}|$ is the number of rounds $t$ such that $p_t = p$, and $\tau_p$ is the distribution of $y_t$ when $t$ is sampled uniformly from $\{t\in [T]:p_t = p\}$.

We use $\CAL_{\mathsf{median}}$, $\CAL_{\mathsf{mean}}$, and $\CAL_{q\text{-}\mathsf{quantile}}$ to denote the swap regret defined by $S_{\mathsf{median}}$, $S_{\mathsf{mean}}$, and $S_{q\text{-}\mathsf{quantile}}$, respectively. 
In particularly, $\CAL_{\mathsf{mean}}$ is the standard $\ell_2$ calibration error:
\begin{equation}
\label{eq:l2_calibration_error}
\CAL_{\mathsf{mean}}(p_1,\ldots,p_T;y_1,\ldots,y_T)
~=~
\sum_{p\in P} n_p\Big(p-\E_{y\sim \tau_p}[y]\Big)^2. 
\end{equation}

The calibration error $\CAL_S$ in \eqref{eq:cal-subopt} can be naturally formulated as swap regret:
\[
\CAL_S(p_{1:T};y_{1:T}) = \sup_{\sigma:[0,1]\to [0,1]}\left(\sum_{t = 1}^TS(p_t,y_t) - \sum_{t=1}^T S(\sigma(p_t),y_t)\right).
\]
Thus our \Cref{thm:main} immediately implies the following corollary:

\begin{corollary}
\label[corollary]{cor:elicit}
    Let $S:[0,1]\times Y\to \R$ be a scoring function such that $S(p,y)$ is a convex and $1$-Lipschitz function of $p\in [0,1]$ for every $y\in Y$. For every positive integer $T$, there exists a randomized prediction strategy that guarantees
    \[
    \E[\CAL_S(p_{1:T};y_{1:T})] = O(\sqrt T \log T).
    \]
As a result, there exists a randomized prediction strategy that guarantees $O(\sqrt T \log T)$ for $\CAL_{\mathsf{median}}$ and $\CAL_{q\text{-}\mathsf{quantile}}$\footnote{$\CAL_{\mathsf{mean}}$ corresponds to a \emph{strongly} convex and smooth loss function, where $\Tilde{O}\left(T^{1/3}\right)$ can be guaranteed \citep{full-swap}. }. 

\end{corollary}
The guarantee above is near-optimal for median calibration. 
\begin{lemma}
\label[lemma]{lem: lowerbound median}
    Suppose $y_1,\ldots,y_T$ are drawn i.i.d.\ from the Bernoulli distribution with mean $1/2$. For any prediction algorithm, $\E[\CAL_{\mathsf{median}}(p_{1:T},y_{1:T})] = \Omega(\sqrt{T})$. 
\end{lemma}
\Cref{lem: lowerbound median} can be proved by noting that $\E[|(y_1 + \cdots + y_T) - T/2|] =\Omega(\sqrt T)$, implying that $\E[\inf_{p\in [0,1]}\sum_{t = 1}^TS(p,y_t)] \le T/2 - \Omega(\sqrt T)$ for scoring rule $S = S_{\mathsf{median}}$. However, we always have $\E[S(p_t,y_t)]\allowbreak = 1/2$. Therefore, 
\[
\E[\CAL_{\mathsf{median}}(p_{1:T},y_{1:T})] \ge \E\left[\sum_{t = 1}^TS(p_t,y_t) - \inf_{p\in [0,1]}\sum_{t = 1}^T S(p,y_t)\right] \ge \Omega(\sqrt T).
\]

We note that results in prior work on calibration for elicitable properties (e.g.\ \cite{scope-multicalibration,swap-multicalibration}) do not apply to median calibration unless additional Lipschitzness assumptions about the distribution of $y$ are satisfied. 
Previous work also considers a different definition of calibration error for elicitable properties.
See \Cref{apdx: calibration for elicitable properties} for a detailed discussion.

\subsection{Technical Overview}
\label{sec:tech}
We give a high-level explanation for how we prove \Cref{thm:main}, highlighting the core ideas. Our main technical innovation lies in the multi-scale binning idea we discuss below.
\paragraph{V-shaped Decomposition.} 
In each round $t$,
the loss function $\ell_t:[0,1]\to \R$ we receive from the adversary  is convex and  $1$-Lipschitz. It is a commonly used fact that such losses can be written as a convex combination of \emph{V-shaped losses} up to an additive constant: there exists a distribution $\varphi_t$ on $[0,1]$ and a constant $C_t\in \R$ such that $\ell_t(p) = \E_{v\sim \varphi_t}|p - v| + C_t$ for every $p\in [0,1]$. This V-shaped decomposition has been established and used in many prior works \citep{opt-scoring-rule,ucal,omni-regression,hw24}. 
We formally state this result in \Cref{lm:vshape} and include a proof for completeness.

This decomposition suggests that we focus on the \emph{base} V-shaped losses of the form $\ell_t(p) = |p - v_t|$ parameterized by $v_t\in [0,1]$. In this subsection, we will assume that every loss $\ell_t$ has this form. This allows us to reason about points $v_t\in [0,1]$ rather than functions $\ell_t:[0,1]\to \R$.

\paragraph{Non-constructive Proof via the Minimax Theorem.}
\Cref{def:main} can be viewed as a zero-sum game between the predictor and the adversary. The predictor plays a mixed strategy that determines the prediction $p_t\in [0,1]$ in every round based on the transcript of previous rounds. After seeing the predictor's mixed strategy, the adversary plays a strategy that determines the loss $\ell_t$ in each round also based on the transcript of previous rounds. 
The value of the game is the expected swap regret when both players play optimally, with the predictor minimizing the expected swap regret and the adversary maximizing it. If we ignore the running time guarantee in \Cref{thm:main}, the theorem is equivalent to the claim that the value of this zero-sum game is $O(\sqrt T\log T)$.

The minimax theorem tells us that the value of a two-player zero-sum game is invariant to who plays first.
We thus consider
an order-reversed game, where the adversary first plays a mixed strategy to choose $\ell_t$, and after seeing this mixed strategy, the predictor then decides their strategy of choosing $p_t$. 
It is easier to construct a good predictor for this order-reversed game because the predictor knows more information than in the original game.
In \Cref{sec:non-construct} we design a prediction strategy that guarantees $O(\sqrt T\log T)$ expected swap regret for the order-reversed game, implying a non-constructive version of \Cref{thm:main} that does not have the running time guarantee. To fully prove \Cref{thm:main}, we apply a multi-objective learning framework to reset the order of play in an efficient way (\Cref{sec:efficient}). 
We further explain how we prove both the non-constructive and original versions of \Cref{thm:main} below.

\paragraph{Binning.}
When minimizing swap regret over a continuous prediction space $[0,1]$, it is natural to restrict the predictions to a finite subset $B\subseteq [0,1]$ of bins. This is beneficial because when two predictions $p_{t}$ and $p_{t'}$ are very close, making them exactly equal ($p_t = p_{t'}$) adds a constraint $\sigma(p_t) = \sigma(p_{t'})$ to  the swap function $\sigma$ in \eqref{eq:sr}, which is helpful for reducing the swap regret. Binning is also helpful when we apply the minimax theorem which requires the action spaces of both players to be finite.

Once we restrict all predictions $p_t$ to a finite set $B\subseteq [0,1]$, we can decompose the swap regret in \eqref{eq:sr} by contribution from each bin $b\in B$. Define $T_b:= \{t\in [T]\,|\,p_t = b\}$ for every $b\in B$, then
\begin{equation}
\label{eq:sb}
\sr(p_{1:T};\ell_{1:T}) = \sum_{b\in B}\left(\sum_{t\in T_b}\ell_t(b) - \sum_{t\in T_b}\ell_t(s_b)\right), \quad \text{where }s_b\in \argmin_{s\in [0,1]}\sum_{t\in T_b}\ell_t(s).
\end{equation}

\paragraph{Truthful Predictor.} 
As we will discuss soon, it is a non-trivial task to choose the set of bins $B$ appropriately. Once $B$ is chosen, a natural prediction strategy for the order-reversed game is to make \emph{truthful} predictions.

Recall that the adversary plays first in the order-reversed game and they can use a mixed strategy. In each round $t = 1,\ldots,T$, the adversary first provides a distribution $\pi_t$ of loss functions $\ell$. Based on $\pi_t$, the predictor chooses $p_t\in B$. After that, a loss function $\ell_t$ is drawn from $\pi_t$ and the predictor incurs loss $\ell_t(p_t)$.

Knowing the loss distribution $\pi_t$, a truthful predictor simply chooses $p_t\in B$ to minimize the expected loss $\bar\ell_t(p_t)$, where $\bar\ell_t:= \E_{\ell\sim\pi_t}[\ell]$.
We analyze the swap regret of this truthful prediction strategy as follows.

\paragraph{Sampling error and rounding error.} 

Let us focus on an arbitrary fixed bin $b\in B$. We define loss functions averaged over rounds $t\in T_b$:
\begin{equation*}
\bar \ell:= \frac 1{|T_b|}\sum_{t\in T_b}\bar \ell_t, \quad \text{and}\quad \hat \ell:= \frac 1{|T_b|}\sum_{t\in T_b} \ell_t.
\end{equation*}
By \eqref{eq:sb}, the swap regret contribution $\sr_b$ from bin $b$ is the following:
\[
\frac 1{|T_b|}\sr_b = \frac 1{|T_b|}\left(\sum_{t\in T_b}\ell_t(b) - \sum_{t\in T_b}\ell_t(s_b)\right) = \hat \ell(b) - \hat \ell(s_b), \quad \text{where }
s_b \in \argmin_{s\in [0,1]}\hat \ell(s).
\]
Analogously to the definition of $s_b$, we define
\[
\bar s_b \in \argmin_{s\in [0,1]}\bar \ell(s).
\]
The truthful prediction strategy ensures that $b$ is the minimizer of $\bar \ell_t(b')$ over $b'\in B$ for every $t\in T_b$. Therefore,
\[
b\in \argmin_{b'\in B}\bar \ell(b').
\]
Both $b$ and $\bar s_b$ are minimizers of the function $\bar \ell$. The difference is that $b$ comes from a restricted set $B$, whereas $\bar s_b$ is chosen from the full domain $[0,1]$. Thus we view the difference between $b$ and $\bar s_b$ as \emph{rounding error} caused by binning.

The swap regret $\sr_b$ contributed by bin $b$ can be decomposed as follows:
\begin{align*}
\frac 1{|T_b|}\sr_b 
= \hat \ell(b) - \hat \ell(s_b)
= {} & \Big(\hat \ell(b) - \hat \ell(s_b)\Big) - \Big(\bar \ell(b) - \bar \ell(s_b)\Big) + \Big(\bar \ell(b) - \bar \ell(s_b)\Big)\\
\le {} & \underbrace{\Big(\hat \ell(b) - \hat \ell(s_b)\Big) - \Big(\bar \ell(b) - \bar \ell(s_b)\Big)}_{\text{Sampling error between $\hat \ell$ and $\bar \ell$}}\quad  + \underbrace{\Big(\bar \ell(b) - \bar \ell(\bar s_b)\Big)}_{\text{Rounding error between $b$ and $\bar s_b$}}.
\end{align*}
There are two sources of error: 1) sampling error between $\hat \ell$ and $\bar \ell$, and 2) rounding error between $b$ and $\bar s_b$. Roughly speaking, increasing the number of bins reduces the rounding error, but it results in fewer data points per bin, causing the sampling error to increase. Our choice of the bins $B$ should minimize the sum of the two types of error.

\paragraph{Challenge with Fixed Binning.} The key technical challenge is that no fixed choice of $B$ can guarantee $\tilde O(\sqrt T)$ swap regret for the truthful prediction strategy. 

Consider the case where all $\pi_t$'s are the same singleton distribution on a single V-shaped loss function $\ell(p) = |p - v|$.
This distribution is deterministic, so we always have $\hat \ell = \bar \ell$ and there is no sampling error. However, the \textbf{rounding error} can be very large: $\bar \ell(\bar s_b) = 0$ is the minimum value of $\bar \ell$, so the rounding error $\bar\ell(b) - \bar \ell(\bar s_b)  = |b - v|$ is the absolute difference between $b$ and $v$. 
To ensure $\widetilde O(\sqrt T)$ swap regret, we need $|b - v| = \widetilde O(\sqrt T)/T = 1 / \widetilde\Omega(\sqrt T)$, but
in the worst case, $v$ can fall at the midpoint of two adjacent bins,
so the gap between any two adjacent bins should be at most $1/\widetilde \Omega(\sqrt T)$. In particular, there needs to be at least $\widetilde \Omega(\sqrt T)$ bins in total.

Now suppose there are $m$ bins in $B$, and assume without loss of generality that at least $m/2$ of these bins $b$ fall in the first half $[0,1/2]$ of the unit interval. Consider the case where for every such bin $b\le 1/2$ and every $t\in T_b$, we choose $\pi_t$ to be a near-uniform distribution on two losses  $\ell(p) = |p - b|$ and $\ell'(p) = |p - (b + 1/2)|$ assigning slightly larger probability mass $1/2 + \varepsilon$ to $\ell$ than the probability mass $1/2 - \varepsilon$ on $\ell'$. The expected loss $\bar \ell:= (1/2 + \varepsilon)\ell + (1/2 - \varepsilon)\ell'$ is indeed minimized at $b$, so the truthful predictor predicts $p_t = b$ and there is no rounding error. However, the \textbf{sampling error} can be very large. The loss $\hat \ell$ can be decomposed as $\eta \ell + (1 - \eta)\ell'$ with $\eta$ being the head frequency of $|T_b|$ independent tosses of a coin with bias $1/2 + \varepsilon$. As $\varepsilon \to 0$, with  probability $\Theta (1)$, we have $\eta \le 1/2 - \Omega(1/\sqrt{|T_b|})$, in which case the minimizer $s_b$ of $\hat \ell$ is $b + 1/2$ instead of $b$, and the difference $\hat \ell(b) - \hat \ell(s_b)$ is $\Omega(1/\sqrt{|T_b|})$. Therefore, the expected swap regret contribution from bin $b$ is $\sr_b = \Omega(\sqrt{T_b})$. Assuming that the bins have roughly equal sizes $|T_b|\approx T/m$, the total swap regret is $\Omega(m\cdot \sqrt{T/m}) = \Omega(\sqrt{Tm})$. Thus we need as few as $m = \mathrm{polylog}(T)$ bins to ensure $\widetilde O(\sqrt T)$ swap regret, contradicting with the earlier requirement of having at least $\widetilde \Omega(\sqrt T)$ bins.

\paragraph{Multi-scale Binning.} 
Our main technical contribution is to use \emph{multi-scale binning} to address the challenge above. The idea is to adapt the bins $B$ to the distribution $\pi_t$ in each round $t$.
For example, in the case above where $\pi_t$ is deterministic on $\ell(p) = |p - v|$, we use as many as $\tilde \Omega(\sqrt T)$ fine-grained bins to reduce the rounding error. Instead, when $\pi_t$ is the uniform distribution on  $\ell(p) = |p - v|$ and $\ell'(p) = |p - (v + 1/2)|$, we use as few as $O(1)$ coarse-grained bins to reduce the sampling error.

In general, the granularity of binning needs to be carefully determined based on the distribution $\pi_t$. To that end, we introduce the notion of \emph{width} $w_t$ of a distribution $\pi_t$ and choose the binning granularity to match $w_t$. Assuming that every loss $\ell\sim \pi_t$ is a V-shaped loss $\ell(p) = |p- v|$, we interpret $\pi_t$ as a distribution of $v\in [0,1]$. Intuitively, when $\pi_t$ is deterministic and supported on a single $v\in [0,1]$, it has small width and we should use fine-grained bins, whereas when $\pi_t$ is uniform on $\{v,v + 1/2\}$, it has large width and we should use coarse-grained bins.  Our exact definition of width for a general distribution on $[0,1]$ is based on a carefully constructed equation about the distance between two quantiles of $\pi_t$ (\Cref{lm:width}).  This definition allows us to optimally balance the rounding and sampling error and show the $O(\sqrt T\log T)$ expected swap regret bound for the order-reversed game. Since the number of predictions that fall in each bin can differ substantially, our full proof in \Cref{sec:non-construct} involves a case analysis that handles heavy and light bins separately.

\paragraph{Efficient Algorithm via Multi-Objective Learning.}
To convert our multi-scale binning predictor for the order-reversed game to an efficient predictor for the original \Cref{def:main}, we reduce \Cref{def:main} to a multi-objective learning task based on the concentration inequalities we used when proving the swap regret guarantee for the order-reversed game. Each concentration inequality corresponds to a constraint or an ``objective'' in multi-objective learning, and our goal is to make predictions to satisfy all of them. This is inspired by a line of work that solves similar problems (e.g.\ online calibration, multicalibration, and omniprediction) using multi-objective learning, such as the Blackwell approachability framework of \cite{blackwell} and the \emph{Adversary-Moves-First (AMF)} framework of \cite{amf}. A key ingredient needed to solve multi-objective learning tasks is an \emph{expert algorithm} that computes a convex combination of the ``objectives'' in every round. To obtain our desired swap regret bound, we  use the MsMwC algorithm of \cite{impossible} which provides a stronger guarantee than the usual multiplicative weights algorithm, following \cite{noarov2025highdimensional,hw24,swap-multicalibration}.
Our final efficient predictor makes randomized predictions over multiple bins across different scales at every round.
See \Cref{sec:efficient} for our full solution to \Cref{def:main}.

\subsection{Related Work}

We defer additional related work to Appendix \ref{sec:relatedwork}.

\section{Preliminaries}

The following standard lemma decomposes any $1$-Lipschitz convex loss $\ell:[0,1]\to \R$ into a convex combination of V-shaped losses. 
\begin{lemma}[\cite{opt-scoring-rule,ucal}]
\label[lemma]{lm:vshape}
    Let $\ell:[0,1]\to \R$ be a $1$-Lipschitz convex function. There exists a distribution $\varphi$ on $[0,1]$ and a constant $C =  (\ell(0) + \ell(1) - 1)/2$ such that
    \[
    \ell(p) = \E_{v\sim \varphi}|p - v| + C \quad \text{for every $p\in [0,1]$.}
    \]
\end{lemma}
\begin{proof}
By replacing $\ell(p)$ with $\ell(p) - C$, we can assume without loss of generality that $\ell(0) + \ell(1) = 1$.
    Since $\ell$ is convex and $1$-Lipschitz, its subgradient $\nabla \ell(p)\in [-1,1]$ is a non-decreasing function of $p\in [0,1]$. We can additionally assume without loss of generality that $\nabla \ell$ is right continuous on $[0,1)$ because we can replace $\nabla \ell(p)$ by $\lim_{p'\to p^+}\nabla \ell(p')$ while keeping $\nabla \ell(p)$ as a subgradient of $\ell$ at $p$. We can also assume without loss of generality that $\nabla \ell(1) = 1$.
    
    We let $\varphi$ be the distribution on $[0,1]$ whose CDF is $(\nabla \ell(p) + 1)/2$. For every $p,v\in [0,1]$,
    \[
    |p - v| = \int_0^1\I[v \le t < p]\mathrm dt + \int_0^1\I[p\le t < v]\mathrm dt.
    \]
Therefore,
\begin{align*}
    \E_{v\sim \varphi}|p - v| &= \int_0^1\E_{v\sim \varphi}\I[v \le t < p]\mathrm dt + \int_0^1 \E_{v\sim \varphi}\I[p\le t < v]\mathrm dt\\
    & = \int_0^p\Pr_{v\sim \varphi}[v \le t]\mathrm dt + \int_p^1\Pr_{v\sim \varphi}[v > t]\mathrm dt\\
    & = \int_0^p \frac{\nabla \ell(t) + 1}{2}\mathrm dt + \int_p^1 \frac{1 - \nabla \ell(t)}{2}\mathrm dt\\
    & = (\ell(p) - \ell(0) + p)/2 + ((1 - p) - \ell(1) + \ell(p))/2\\
    & = \ell(p) + (1 - \ell(0) - \ell(1)) / 2\\
    & = \ell(p). \qedhere
\end{align*}
\end{proof}

As we discuss in \Cref{sec:tech} and formally prove in \Cref{sec:strategy}, the minimax theorem allows us to consider the following order-reversed game in lieu of the original swap regret minimization problem (\Cref{def:main}):
\begin{problem}[Order-reversed online swap regret minimization]
\label[problem]{def:order-reversed}
Let $T$ be a positive integer and $B$ be a finite subset of $[0,1]$. We study the following sequential prediction problem.
For round $t = 1,\ldots,T$,
\begin{enumerate}
    \item adversary reveals a distribution $\pi_t$ of $1$-Lipschitz convex losses $\ell:[0,1]\to \R$;
    \item predictor chooses $p_t\in B\subseteq[0,1]$;
    \item loss $\ell_t\sim \pi_t$ is sampled and revealed, and predictor incurs loss $\ell_t(p_t)$.
\end{enumerate}
The predictor's goal is to minimize the expected swap regret $\E[\sr(p_{1:T}, \ell_{1:T})]$ defined in \eqref{eq:sr}.
\end{problem}

\section{Non-constructive Proof}
\label{sec:non-construct}

We show a prediction strategy that achieves $O(\sqrt T \log T)$ expected swap regret for the order-reversed game in \Cref{def:order-reversed} (\Cref{lm:truthful-reversed}). By the minimax argument we formally describe in \Cref{sec:strategy}, this implies the same expected swap regret guarantee for \Cref{def:main}, thus proving a non-constructive version of \Cref{thm:main} that removes the running time guarantee.

\subsection{Width of a Distribution}
We introduce the notion of \emph{width} for distributions on the unit interval $[0,1]$ in \Cref{lm:width} below. This notion plays a central role in our multi-scale binning idea for solving \Cref{def:order-reversed}.
\begin{definition}
\label[definition]{def:quantile}
    Let $\varphi$ be an arbitrary distribution on $[0,1]$. For $z\in [0,1]$, we say $q\in [0,1]$ is a $z$-quantile of $\varphi$ if
    \[
    \Pr_{v\sim \varphi}[v < q] \le z \quad \text{and} \quad \Pr_{v\sim \varphi}[v \le q] \ge z.
    \]
    We use $Q_\varphi(z)\subseteq [0,1]$ to denote the set of all $z$-quantiles of $\varphi$.
\end{definition}

Intuitively, if we eventually choose bins of width $w$, there are two things we want to balance. First is the width $w$ itself, since it contributes directly to the rounding error. The second is the reciprocal of the probability mass in a median-centered interval of this width, since this provides a rough estimate of the number of bins we expect to play (and therefore contributes to the sampling error). The following definition of width balances these two quantities. 

\begin{lemma}
\label[lemma]{lm:width}
    Let $\varphi$ be a distribution on $[0,1]$ and let $\gamma\in (0,1]$ be an arbitrary parameter. There exists a unique value $w\in [\gamma,1]$ such that
    \[
    w\in Q_\varphi\left(\frac 12 + \frac {\gamma}{2w}\right) - Q_\varphi\left(\frac 12 - \frac {\gamma}{2w}\right).
    \]
Here, the difference $Q_1 - Q_2$ between two sets $Q_1$ and $Q_2$ is defined as the set $\{q_1 - q_2\,|\,q_1\in Q_1,q_2\in Q_2\}$. That is, there exist $\alpha,\beta\in [0,1]$ such that
\[
\alpha\in Q_\varphi\left(\frac 12 - \frac {\gamma}{2w}\right),\quad \beta \in Q_\varphi\left(\frac 12 + \frac {\gamma}{2w}\right), \quad \beta - \alpha = w.
\]
We say $w$ is the $\gamma$-\emph{width} of $\varphi$.
\end{lemma}
To prove \Cref{lm:width}, we need the following lemma that follows from basic properties of $\varphi$ as a probability measure.
\begin{lemma}
\label[lemma]{lm:measure}
Let $\varphi$ be a distribution on $[0,1]$. For every $z\in [0,1]$,
    $Q_\varphi(z)\subseteq [0,1]$ is a non-empty closed interval or a single point. Moreover, 
    \begin{align*}
    \max Q_\varphi(z) & = \begin{cases}1, & \text{if }z = 1,\\
    \inf_{z'\in (z,1]}\min Q_\varphi(z'), & \text{if }z\in [0,1);
        \end{cases}\\
    \min Q_\varphi(z) & = \begin{cases}0, & \text{if }z = 0,\\
    \sup_{z' \in [0,z)}\max Q_\varphi(z'), & \text{if }z\in (0,1].
        \end{cases}
    \end{align*}
\end{lemma}
\begin{proof}[Proof of \Cref{lm:width}]
For every $w\in [\gamma,1]$, define
\[
S(w):= Q_\varphi\left(\frac 12 + \frac {\gamma}{2w}\right) - Q_\varphi\left(\frac 12 - \frac {\gamma}{2w}\right).
\]
By \Cref{lm:measure}, for every $w\in [\gamma,1]$, $S(w)$ is a non-empty closed interval or a single point, and $S(w)\subseteq[0,1]$.  
Our goal is to show that there exists a unique $w^*\in [\gamma,1]$ such that $w^*\in S(w^*)$. By \Cref{lm:measure},
    \begin{align}
    \max S(w) & = \begin{cases}1, & \text{if }w = \gamma,\\
    \inf_{w' \in [\gamma,w)}\min S(w'), & \text{if }w\in (\gamma,1];
        \end{cases} \label{eq:Sw-1}\\
    \min S(w) & = \sup_{w' \in (w,1]}\max S(w'), \quad \text{if }w\in [\gamma,1).\label{eq:Sw-2}
    \end{align}
Define $W:=\{w\in [\gamma,1] \, |\, w \le \max S(w)\}$.
By \eqref{eq:Sw-1}, we have $\max S(\gamma) = 1$, so $\gamma\in W$. In particular, $W$ is a non-empty set. Let $w^*:= \sup W$. We show that $w^*\in S(w^*)$.

We first show $w^* \le \max S(w^*)$. This holds trivially if $w^* = \gamma$. When $w^* > \gamma$, by \eqref{eq:Sw-1}, for every $w\in (\gamma,w^*)\cap W$, we have $w \le \max S(w) = \inf _{w'\in [\gamma,w)}\min S(w')$. Taking the limit $w\to w^*$, we get $w^*\le \inf_{w'\in [\gamma,w^*)}\min S(w') = \max S(w^*)$.

Now we show $w^* \ge \min S(w^*)$ This holds trivially if $w^* = 1$. When $w^* < 1$, by the definition of $w^* = \sup W$, for every $w\in (w^*,1]$, we have $w > \max S(w)$. Taking the limit $w \to w^*$, we have $w^* \ge \lim_{w\to (w^*)^+}\max S(w) = \sup_{w \in (w^*,1]}\max S(w)  = \min S(w^*)$ by \eqref{eq:Sw-2}. 

We have proved that $\min S(w^*) \le w^* \le \max S(w^*)$, which implies that $w^*\in S(w^*)$. To show that $w^*$ is the unique value with this property, we note that for every $w < w^*$, we have 
\[
w < w^* \le \max S(w^*) = \inf_{w' < w^*} \min S(w') \le \min S(w),
\]
which means that $w\notin S(w)$. Similarly, for every $w > w^*$, we also have $w\notin S(w)$.
\end{proof}

\subsection{Truthful Predictor with Multi-scale Binning}
\label{sec:truthful}
We are now ready to describe our prediction strategy that achieves $O(\sqrt T\log T)$ expected swap regret for \Cref{def:order-reversed}. In fact, we prove a stronger result: given any $\delta\in (0,1/T]$ as input, our strategy guarantees $O\big(\sqrt{(T\log T)\log(1/\delta)}\big)$ swap regret with probability at least $1-\delta$.

Let $\gamma\in (0,1]$ be a parameter we determine later. We define $R$ as the set of bin scales we consider:
\begin{equation}
    \label{eq:R}
    R:= \{\gamma, 2\gamma, 4\gamma, 8\gamma,\ldots\}\cap [0,1].
\end{equation}
We will always ensure $\gamma = \Omega(1/\sqrt T)$, which implies $|R| = O(\log T)$. For every $r\in R$, we define $B_r$ as the bins at scale $r$:
\begin{equation}
    \label{eq:Br}
    B_r:=\{0, r, 2r, 3r,\ldots\}\cap [0,1].
\end{equation}
Clearly, $|B_r| = O(1/r)$. We define 
\begin{equation}
\label{eq:Theta}
B:= \bigcup\nolimits_{r\in R}B_r \quad \text{and} \quad \Theta:= \{(r,b):r\in R,b\in B_r\}. 
\end{equation}
We have 
$
|B|\le |\Theta| = \sum_{r\in R}|B_r| = O(1/\gamma) = O(\sqrt T).
$
We use the following truthful predictor to achieve the $O\big(\sqrt{(T\log T)\log(1/\delta)}\big)$ swap regret bound for \Cref{def:order-reversed}.
\paragraph{Truthful Predictor with Multi-scale Binning.}
We choose $\gamma := \sqrt{(\ln T)(\ln(1/\delta))/T}$. If $\gamma > 1$, the trivial swap regret upper bound $T$ already implies the desired $O\big(\sqrt{(T\log T)\log(1/\delta)}\big)$ bound. We thus assume without loss of generality that $\gamma \le 1$.
In each round $t = 1,\ldots,T$, we take the following steps:
\begin{enumerate}
    \item We receive a distribution $\pi_t$ from the adversary. Here $\pi_t$ is a distribution of $1$-Lipschitz convex loss functions $\ell:[0,1]\to \R$. By \Cref{lm:vshape}, each $\ell$ corresponds to a distribution $\varphi$ on $[0,1]$. Thus we can equivalently view $\pi_t$ as a meta-distribution of distributions $\varphi$ on $[0,1]$. We define distribution $\bar \varphi_t$ on $[0,1]$ as the mixture of distributions $\varphi$ drawn from $\pi_t$.
    \item We define $w_t\in [\gamma,1]$ to be the $\gamma$-width of $\bar \varphi_t$ (\Cref{lm:width}).
This implies the existence of $\alpha_t,\beta_t\in [0,1]$ such that
\[
\alpha_t\in Q_{\bar \varphi_t}\left(\frac 12 - \frac{\gamma}{2w_t}\right),\quad \beta_t \in Q_{\bar \varphi_t}\left(\frac 12 + \frac{\gamma}{2w_t}\right), \quad \beta_t - \alpha_t = w_t.
\]
\item We let $r_t$ be the unique value in $R$ such that $w_t\in [r_t,2r_t)$. Since $\beta_t - \alpha_t = w_t \ge r_t$, there exists $b_t\in B_{r_t}\subseteq B$ such that $b_t\in [\alpha_t,\beta_t]$. We output $p_t = b_t$.
\item We observe loss $\ell_t$ drawn from $\pi_t$, and let $\varphi_t$ be the distribution on $[0,1]$ corresponding to $\ell_t$ by \Cref{lm:vshape}.
\end{enumerate}

\begin{lemma}
\label[lemma]{lm:truthful-reversed}
    For every integer $T\ge 2$ and every $\delta\in (0,1/T]$, the truthful predictor above achieves $O\big(\sqrt{(T\log T)\log(1/\delta)}\big)$ swap regret with probability at least $1-\delta$ for \Cref{def:order-reversed}. In particular, setting $\delta = 1/T$, we get an $O(\sqrt T\log T)$ expected swap regret guarantee.
\end{lemma}
To prove \Cref{lm:truthful-reversed}, we need the following basic lemma about the median of a random variable as the minimizer of the expected distance.
\begin{lemma}
\label[lemma]{lm:suboptimality}
    Let $\varphi$ be a distribution of $v\in [0,1]$, and let $s\in Q_\varphi(1/2)$ be a median of $\varphi$. For every $b\in [0, 1]$, we have
    \[
    0 \le \E_{v\sim \varphi}|b - v| - \E_{v\sim \varphi} |s - v| \le |b - s|(1 - 2 \Pr\nolimits_{v\sim \varphi}[s < b \le v \text{ or } v \le b < s]).
    \]
\end{lemma}
\begin{proof}[Proof of \Cref{lm:suboptimality}]
    We assume without loss of generality that $b > s$. The other case $b < s$ can be handled similarly by symmetry. We have
    \[
    |b - v| - |s - v| \le \begin{cases}
        b - s,& \text{if } v < b;\\
        s - b, & \text{if }v \ge b.
    \end{cases}
    \]
    Therefore, 
    \[
    \E_{v\sim \varphi}[|b - v| - |s - v|] \le (b - s)(\Pr\nolimits_{v\sim \varphi}[v < b] - \Pr\nolimits_{v\sim \varphi}[v \ge b]) = (b - s)(1 - 2\Pr\nolimits_{v\sim \varphi}[v \ge b]).
    \]
    We also have
        \[
    |b - v| - |s - v| \ge \begin{cases}
        b - s,& \text{if } v \le s;\\
        s - b, & \text{if }v > s.
    \end{cases}
    \]
    Therefore, 
    \[
    \E_{v\sim \varphi}[|b - v| - |s - v|] \ge (b - s)(\Pr\nolimits_{v\sim \varphi}[v \le s] - \Pr\nolimits_{v\sim \varphi}[v > s]) \ge  (b - s)(1/2 - 1/2) = 0. \qedhere
    \]
\end{proof}

\begin{proof}[Proof of \Cref{lm:truthful-reversed}]
Recall $\gamma = \sqrt{(\ln T)(\ln(1/\delta))/T}$. If $\gamma > 1$, we can simply use the trivial swap regret upper bound $T$ to get $\sr(p_{1:T};\ell_{1:T}) \le T = O\big(\sqrt{(T\log T)\log(1/\delta)}\big)$. We thus assume $\gamma \le 1$ from now on.

For every pair $(r,b)\in \Theta$, we let $T_{r,b}$ denote the set of time steps $t\in \{1,\ldots,T\}$ such that $(r_t,b_t) = (r,b)$. For every $t\in T_{r,b}$, we have
\begin{align*}
\Pr_{v\sim \bar \varphi_t}[v > b + 2r] & \le \Pr_{v\sim \bar \varphi_t}[v > \beta_t] \le  \frac 12 - \frac{\gamma}{2w_t} \le  \frac 12 - \frac{\gamma}{4r},\\
\Pr_{v\sim \bar \varphi_t}[v < b - 2r] & \le  \Pr_{v\sim \bar \varphi_t}[v > \beta_t] \le  \frac 12 - \frac{\gamma}{2w_t} \le  \frac 12 - \frac{\gamma}{4r},\\
\Pr_{v\sim \bar \varphi_t}[v \ge b] & \ge  \Pr_{v\sim \bar \varphi_t}[v > \beta_t] \ge  \frac 12 - \frac{\gamma}{2w_t} \ge  \frac 12 - \frac{\gamma}{2r},\\
\Pr_{v\sim \bar \varphi_t}[v \le b] & \ge \Pr_{v\sim \bar \varphi_t}[v > \beta_t] \ge  \frac 12 - \frac{\gamma}{2w_t} \ge  \frac 12 - \frac{\gamma}{2r}.
\end{align*}
Let $\hat \varphi_{r,b}$ be the uniform mixture of the distributions $\varphi_t$ for $t\in T_{r,b}$. Note that each $\varphi_t$, as a random variable, has mean $\bar \varphi_t$. Thus by standard martingale concentration bounds (Azuma's inequality, see \Cref{lm:azuma}), with probability at least $1 - \delta^4/100$, 
\begin{align}
\Pr_{v\sim \hat \varphi_{r,b}}[v > b + 2r] & \le  \frac 12 - \frac{\gamma}{4r} + O\left(\sqrt{\frac{\log (1/\delta)}{|T_{r,b}|}}\right), \label{eq:phi-hat-1}\\
\Pr_{v\sim \hat \varphi_{r,b}}[v < b - 2r] & \le \frac 12 - \frac{\gamma}{4r} + O\left(\sqrt{\frac{\log (1/\delta)}{|T_{r,b}|}}\right), \label{eq:phi-hat-2}\\
\Pr_{v\sim \hat \varphi_{r,b}}[v \ge b] & \ge \frac 12 - \frac{\gamma}{2r} - O\left(\sqrt{\frac{\log (1/\delta)}{|T_{r,b}|}}\right), \label{eq:phi-hat-3}\\
\Pr_{v\sim \hat \varphi_{r,b}}[v \le b] & \ge \frac 12 - \frac{\gamma}{2r} - O\left(\sqrt{\frac{\log (1/\delta)}{|T_{r,b}|}}\right).\label{eq:phi-hat-4}
\end{align}
By the union bound, with probability at least $1-1/\delta$, these inequalities hold for all pairs $(r,b)\in \Theta$. It remains to show that whenever these inequalities hold, the swap regret is $ O\big(\sqrt{(T\log T)\log(1/\delta)}\big)$.
The following equation follows from the definition of $\gamma := \sqrt{(\ln T)(\ln(1/\delta))/T}$:
\begin{equation}
    \label{eq:useful-0}
    \ln(1/\delta) = T\gamma^2/\ln T.
\end{equation}

For every $(r,b)\in \Theta$, arbitrarily choose $s_{r,b}\in Q_{\hat \varphi_{r,b}}(1/2)$. We can decompose the swap regret as follows.
\begin{align}
    \sr(p_{1:T},\ell_{1:T}) & = \sum_{t = 1}^T \ell_t(p_t) - \inf_{\sigma:[0,1]\to [0,1]}\sum_{t = 1}^T \ell_t(\sigma(p_t))\notag \\
    & = \sum_{t = 1}^T \ell_t(p_t) - \sum_{b\in B}\inf_{s\in [0,1]} \sum_{t = 1}^T\I[p_t = b]\ell_t(s)\notag \\
    & \le \sum_{t = 1}^T\ell_t(p_t) - \sum_{(r,b)\in \Theta}\inf_{s\in [0,1]} \sum_{t\in T_{r,b}}\ell_t(s)\notag \\
    & = \sum_{(r,b)\in \Theta}\sum_{t\in T_{r,b}}\ell_t(p_t) - \sum_{(r,b)\in \Theta}\inf_{s\in [0,1]} \sum_{t\in T_{r,b}}\ell_t(s)\notag \\
    & = \sum_{(r,b)\in \Theta}\sum_{t\in T_{r,b}}\E_{v\sim \varphi_t}|b - v| - \sum_{(r,b)\in \Theta}\inf_{s\in [0,1]} \sum_{t\in T_{r,b}}\E_{v\sim \varphi_t}|s - v|\notag \\
        & = \sum_{(r,b)\in \Theta}\sum_{t\in T_{r,b}}\E_{v\sim \varphi_t}|b - v| - \sum_{(r,b)\in \Theta}\sum_{t\in T_{r,b}}\E_{v\sim \varphi_t}|s_{r,b} - v|\notag \\
    & = \sum_{(r,b)\in \Theta}|T_{r,b}|\cdot\E_{v\sim \hat \varphi_{r,b}}[|b - v| - |s_{r,b} - v|].\label{eq:decomposition}
\end{align}

Let $C > 0$ be a sufficiently large absolute constant.
We say a pair $(r,b)\in \Theta$ is \emph{heavy} if $|T_{r,b}| \ge CTr^2 / \ln T$. Otherwise, we say $(r,b)$ is \emph{light}. For every fixed heavy pair $(r,b)$, we have
\begin{equation}
\label{eq:heavy-0}
\sqrt{\frac{\ln (1/\delta)}{|T_{r,b}|}} \le \sqrt{\frac{(\ln(1/\delta))(\ln T)}{Cr^2T}} = \frac{\gamma}{r \sqrt C}.
\end{equation}
Therefore, when $C$ is sufficiently large, \eqref{eq:phi-hat-1} and \eqref{eq:phi-hat-2} imply
\[
\Pr_{v\sim \hat \varphi_{r,b}}[v > b + 2r] \le 1/2, \quad \Pr_{v\sim \hat \varphi_{r,b}}[v < b - 2r] \le 1/2,
\]
so there exists $s_{r,b}\in Q_{\hat \varphi}(1/2)$ such that $|s_{r,b} - b|\le 2r$. Similarly, \eqref{eq:phi-hat-3} and \eqref{eq:phi-hat-4} imply
\[
\Pr_{v\sim \hat \varphi_{r,b}}[v \ge b] \le 1/2 - O\left(\frac \gamma r\right), \quad \Pr_{v\sim \hat \varphi_{r,b}}[v \le b] \le 1/2 - O\left(\frac \gamma r\right),
\]
so by \Cref{lm:suboptimality},
\begin{equation}
\label{eq:reversed-heavy}
\E_{v\sim \hat \varphi_{r,b}}[|v - b| - |v - s_{r,b}|] \le 2r \cdot O\left(\frac{\gamma}{r}\right)  = O\left(\gamma\right).
\end{equation}

For every light pair $(r,b)$, we have the reversed version of inequality \eqref{eq:heavy-0}:
\[
\sqrt{\frac{\ln (1/\delta)}{|T_{r,b}|}} \ge \frac{\gamma}{r \sqrt C}.
\]
Therefore, \eqref{eq:phi-hat-3} and \eqref{eq:phi-hat-4} imply
\[
\Pr_{v\sim \hat \varphi_{r,b}}[v \ge b] \le 1/2 - O\left(\sqrt{\frac{\log (1/\delta)}{|T_{r,b}|}}\right), \quad \Pr_{v\sim \hat \varphi_{r,b}}[v \le b] \le 1/2 - O\left(\sqrt{\frac{\log (1/\delta)}{|T_{r,b}|}}\right).
\]
By \Cref{lm:suboptimality} and the trivial bound $|b - s_{r,b}| \le 1$ for an arbitrary $s_{r,b}\in Q_{\hat \varphi_{r,b}}(1/2)$, we have
\[
\E_{v\sim \hat \varphi_{r,b}}[|v - b| -|v - s_{r,b}|] \le O\left(\sqrt{\frac{\log (1/\delta)}{|T_{r,b}|}}\right).
\]
Therefore,
\begin{align}
|T_{r,b}|\cdot \E_{v\sim \hat \varphi_{r,b}}[|v - b| -|v - s_{r,b}|] & = O\left(\sqrt{|T_{r,b}|}\cdot \sqrt{\log(1/\delta)}\right)\notag \\
& = O\left(\sqrt{Tr^2/\log T} \cdot \sqrt{T\gamma^2/\log T}\right)\tag{by \eqref{eq:useful-0}} \\
& = O\left(Tr\gamma/\log T\right).\label{eq:reversed-light}
\end{align}
Combining \eqref{eq:reversed-heavy} and \eqref{eq:reversed-light}, we have
\begin{align*}
    & \sr (p_{1:T};\ell_{1:T})\\
    \le {} &  \sum_{(r,b)\in \Theta}|T_{r,b}|\cdot \E_{v\sim \hat \varphi_{r,b}}[|v - b| - |v - s_{r,b}|]\\
    = {} &  \sum_{\text{heavy }(r,b)}|T_{r,b}|\cdot \E_{v\sim \hat \varphi_{r,b}}[|v - b| - |v - s_{r,b}|] + \sum_{\text{light }(r,b)}|T_{r,b}|\cdot \E_{v\sim \hat \varphi_{r,b}}[|v - b| - |v - s_{r,b}|]\\
    \le {} & O\left(\gamma\right)\cdot \sum_{\text{heavy }(r,b)}|T_{r,b}| + \sum_{\text{light }(r,b)}O\left(Tr\gamma/\log T\right)\\
    \le {} & O\left(\gamma\right)\cdot T + \sum_{r\in R}O\left(T\gamma/\log T\right) \tag{because $|B_r|  = O(1/r)$}\\
    = {} &  O(T\gamma) \tag{because $|R| = O(\log T)$}\\
    = {} & O\left(\sqrt{(T\log T)\log(1/\delta)}\right).\qedhere
    \end{align*}
\end{proof}

\section{Efficient Algorithm}
\label{sec:efficient}

In the previous section, we proved the \emph{existence} of a randomized prediction strategy achieving $O(\sqrt T\log T)$ expected swap regret. We now explicitly construct such a prediction algorithm to obtain the running time guarantee of \Cref{thm:main}.

The challenge is that we can no longer rely on the minimax theorem and work with the order-reversed \Cref{def:order-reversed} where we see the distribution $\pi_t$ of loss functions before making a prediction $p_t$. Instead, we need to directly work with \Cref{def:main}, where we compute a distribution $\rho_t$ of randomized predictions $p_t$ based only on information from previous rounds, with the loss function $\ell_t$ revealed only after we commit to the distribution $\rho_t$. Inspired by the analysis in the previous section, in each round we output not just a randomized prediction $p_t = b_t\in B\subseteq [0,1]$, but also a scale $r_t\in R$. Here we define $R,B_r,B,\Theta$ as in \eqref{eq:R}-\eqref{eq:Theta} in the previous section.
As before, for every $(r,b)\in \Theta$, we define $T_{r,b}:=\{t\in [T]\, |\, (r_t,b_t) = (r,b)\}$, and define $\hat \varphi_{r,b}$ as the uniform mixture of $\varphi_t$ over $t\in T_{r,b}$.
Based on our proof of \Cref{lm:truthful-reversed} in \Cref{sec:non-construct}, it suffices to compute a distribution $\kappa_t$ of $(r_t,b_t)\in \Theta$ in each round such that 
inequalities \eqref{eq:phi-hat-1}-\eqref{eq:phi-hat-4}  hold simultaneously for all pairs $(r,b)\in \Theta$ with probability at least $1-\delta$.

The goal of satisfying the constraints \eqref{eq:phi-hat-1}-\eqref{eq:phi-hat-4} can be naturally formulated as an online multi-objective learning task. We apply ideas in the \emph{Adversary-Moves-First (AMF)} framework of \cite{amf} to extend our prediction strategy for the order-reversed game (\Cref{def:order-reversed}) to the original game (\Cref{def:main}). The idea is to satisfy a different mixture (i.e.\ convex combination) of the constraints at every round, and use an \emph{expert algorithm} to choose the mixtures so that satisfying these mixtures implies satisfying \emph{every} individual constraint in hindsight. Similarly to the work of \cite{hw24,swap-multicalibration} and inspired by \cite{noarov2025highdimensional}, we use the expert algorithm MsMwC from \cite{impossible} to obtain a stronger guarantee than that of the usual multiplicative weights algorithm. This stronger guarantee is critical for achieving our final swap regret bound $O(\sqrt T\log T)$.

We first formulate the constraints \eqref{eq:phi-hat-1}-\eqref{eq:phi-hat-4} using functions, and show that every convex combination of the constraints can be satisfied by some randomized prediction which we can compute efficiently via linear programming (\Cref{lm:exists-tau}).
For every pair $(r,b)\in \Theta$ and $\xi\in \{-2,-1,+1,+2\}$, we define function $h_{r,b,\xi}:\Theta\times [0,1]\to [-1,1]$ as follows. For every $(r',b')\in \Theta$ and $v\in [0,1]$,
\begin{align*}
h_{r,b,+2}(r',b',v) & = \I[(r',b') = (r,b)] \left(\I[v > b + 2r] - \left(\frac 12 - \frac{\gamma}{4r}\right)\right),\\
h_{r,b,-2}(r',b',v) & = \I[(r',b') = (r,b)] \left(\I[v < b - 2r] - \left(\frac 12 - \frac{\gamma}{4r}\right)\right),\\
h_{r,b,+1}(r',b',v) & = \I[(r',b') = (r,b)] \left(\left(\frac 12 - \frac{\gamma}{2r}\right) - \I[v \ge b] \right),\\
h_{r,b,-1}(r',b',v) & = \I[(r',b') = (r,b)] \left( \left(\frac 12 - \frac{\gamma}{2r}\right) - \I[v \le b] \right).
\end{align*}
\begin{lemma}
\label[lemma]{lm:exists-tau}
Let $\gamma\in (0,1]$ be a parameter and define $R,B_r,B,\Theta$ as in \eqref{eq:R}-\eqref{eq:Theta}.
    Let $\bar h:\Theta\times [0,1]\to [-1,1]$ be a convex combination of the functions  $h_{r,b,\xi}$ for $(r,b)\in \Theta,\xi\in \{-2,\allowbreak -1,\allowbreak +1,\allowbreak +2\}$. That is, $\bar h$ is given by weights $w_{r,b,\xi}\ge 0$ satisfying $\sum_{r,b,\xi}w_{r,b,\xi} = 1$:
    \[
    \bar h = \sum_{(r,b)\in \Theta,\xi\in \{-2,-1,1,2\}}w_{r,b,\xi}h_{r,b,\xi}.
    \]
    Then there exists a distribution $\kappa$ of $(r',b')\in \Theta$ such that
    \begin{equation}
    \label{eq:exist-dist}
    \E_{(r',b')\sim \kappa}[\bar h(r',b',v)] \le 0 \quad \text{for every $v\in [0,1]$.}
    \end{equation}
Moreover, given the weights $w_{r,b,\xi}$ as input, the distribution $\kappa$ can be computed in time $\poly(|\Theta|)$ by solving a linear program.
\end{lemma}
We prove \Cref{lm:exists-tau} using \Cref{lm:exists-ir} below and the minimax theorem. 
\begin{lemma}
\label[lemma]{lm:exists-ir}
Let $\gamma\in (0,1]$ be a parameter and define $R,B_r,B,\Theta$ as in \eqref{eq:R}-\eqref{eq:Theta}.
    For every distribution $\varphi$ of $v\in [0,1]$, there exists $(r',b')\in \Theta$ such that
    \[
    \E_{v\sim \varphi}[h_{r,b,\xi}(r',b',v)] \le 0 \quad \text{for every $(r,b)\in \Theta$ and $\xi\in \{-2,-1,+1,+2\}$.}
    \]
\end{lemma}

We prove \Cref{lm:exists-ir} using the truthful predictor we constructed in \Cref{sec:truthful}.
\begin{proof}[Proof of \Cref{lm:exists-ir}]
We apply the truthful predictor in \Cref{sec:truthful}.
    Define $\alpha,\beta,w$ as in \Cref{lm:width} for $\varphi$ and $\gamma$. Choose $r'\in R$ so that $w \in [r',2r')$. Choose $b'\in B_{r'}$ so that $b'\in [\alpha,\beta]$. 
    This ensures that
    \[
    \E_{v\sim \varphi}[h_{r',b', \xi}(r',b',v)] \le 0\quad \text{for every $\xi \in \{-2,-1,+1,+2\}$.}
    \]
The proof is completed by noting that $h_{r,b,\xi}(r',b',v) = 0$ if $(r,b)\ne (r',b')$.
\end{proof}

\begin{proof}[Proof of \Cref{lm:exists-tau}]
We first prove the existence of a distribution $\kappa$ satisfying \eqref{eq:exist-dist}. Although we require \eqref{eq:exist-dist} to hold for every $v\in [0,1]$, we only need to consider a finite set $V\subseteq [0,1]$ of $v$ with size $|V|= O(|\Theta|)$. This is because the functions $h_{r,b,\xi}$ only depend on comparisons between $v$ and a finite set of values $b + 2r, b - 2r$ and  $b$ for $(r,b)\in \Theta$. \

For an arbitrary set $S$, we use $\Delta_S$ to denote the set of probability distributions on $S$.
By \Cref{lm:exists-ir},
\[
\max_{\varphi\in \Delta_V}\min_{(r',b')\in \Theta}\E_{v\sim \varphi}[\bar h(r',b',v)] \le 0.
\]
By the minimax theorem,
\[
\min_{\kappa\in \Delta_\Theta}\max_{v\in V}\E_{(r',b')\sim \kappa}[\bar h(r',b',v)] \le 0.
\]
This implies the existence of distribution $\kappa$ satisfying \eqref{eq:exist-dist} for every $v\in V$, as desired.

To see that $\kappa$ can be computed in $\poly(|\Theta|)$ time, we note that for every $v\in V$, \eqref{eq:exist-dist} is a linear constraint on the probability masses of $\kappa$. Thus we can compute $\kappa$ by solving a linear program with $O(|\Theta|)$ constraints and variables.
\end{proof}

\paragraph{Expert algorithm MsMwC of \cite{impossible}.}
In each round $t = 1,\ldots,T$, our prediction algorithm first computes a convex combination $\bar h_t$ of the constraint functions. It then makes a randomized prediction using the distribution $\tau_t$ from \Cref{lm:exists-tau}. The convex combination $\bar h_t$ is given by weights $w\sps t_{r,b,\xi}\ge 0$ such that $\sum_{r,b,\xi} w\sps t_{r,b,\xi} = 1$:
\begin{equation}
\label{eq:mixture}
\bar h_t = \sum_{r,b,\xi}w_{r,b,\xi}\sps t h_{r,b,\xi}.
\end{equation}
We use the MsMwC algorithm of \cite{impossible} to compute these weights. The algorithm has a low external regret guarantee \eqref{eq:mcmwc} in the following game and runs in time $\poly(T,|\Theta|)$: in each round $t = 1,\ldots,T$,
\begin{enumerate}
    \item The MsMwC algorithm computes $w\sps t_{r,b,\xi}\ge 0$ for every $(r,b)\in \Theta$ and $\xi\in \{-2,-1,+1,+2\}$ such that $\sum_{r,b,\xi} w\sps t_{r,b,\xi} = 1$;
    \item The MsMwC algorithm receives adversarially chosen reward $u\sps t_{r,b,\xi}\in [-1,1]$ for every $(r,b)\in \Theta$ and $\xi\in \{-2,-1,+1,+2\}$.
\end{enumerate}
The MsMwC algorithm guarantees that for every $(r_0,b_0)\in \Theta$ and $\xi_0\in \{-2,-1,+1,+2\}$,
\begin{equation}
\label{eq:mcmwc}
\sum_{i = 1}^T \sum_{r,b,\xi}w\sps t_{r,b,\xi}u\sps t_{r,b,\xi} \ge \sum_{i = 1}^T u\sps t_{r_0,b_0,\xi_0} - O\left(\log T + \sqrt{\sum_{t = 1}^T\left(u\sps t_{r_0,b_0,\xi_0}\right)^2\log T}\right).
\end{equation}

\paragraph{Prediction Algorithm.} 
We are now ready to describe our efficient prediction algorithm for solving \Cref{def:main}. 
As in \Cref{sec:truthful}, we choose $\gamma := \sqrt{(\ln T)(\ln(1/\delta))/T}$ and assume without loss of generality that $\gamma \le 1$. Define $R,B_r,B,\Theta$ as in \eqref{eq:R}-\eqref{eq:Theta}. 
For round $t = 1,\ldots,T$,
\begin{enumerate}
    \item Compute weights $w\sps t_{r,b,\xi}$ using MsMwC and form the convex combination $\bar h_t: \Theta\times [0,1]\to [-1,1]$ by \eqref{eq:mixture};
    \item Compute $\kappa_t$ from $\bar h_t$ as in \Cref{lm:exists-tau};
    \item Sample $(r_t,b_t)\in \Theta$ from $\kappa_t$ and output $p_t:= b_t$;
    \item Observe loss function $\ell_t\in [0,1]$. Let $\varphi_t$ denote the distribution corresponding to $\ell_t$;
    \item For every $(r,b)\in \Theta$ and $\xi\in \{-2,-1,+1,+2\}$, provide reward 
    \[
    u_{r,b,\xi}\sps t:= \E_{(r',b')\sim \tau_t}\E_{v\sim \varphi_t}h_{r,b,\xi}(r',b', v)
    \]
    to the MsMwC algorithm.
\end{enumerate}
\begin{lemma}
\label[lemma]{lm:efficient}
Let $T > 2$ be an integer.
    Given any $\delta\in (0,1/T]$, the prediction algorithm above guarantees 
\[
\sr(p_{1:T};\ell_{1:T}) = O\left(\sqrt{(T\log T)\log(1/\delta)}\right)
\]
    with probability at least $1-\delta$ for \Cref{def:main} and runs in $\poly(T)$ time. In particular, setting $\delta = 1/T$, we get an $O(\sqrt T \log T)$ expected swap regret upper bound.
\end{lemma}
\begin{proof}[Proof of \Cref{lm:efficient}] 
The running time guarantee follows from the corresponding guarantees from \Cref{lm:exists-tau} and the efficiency of the MsMwC algorithm. We prove the swap regret guarantee below.

The MsMwC algorithm ensures \eqref{eq:mcmwc}, which means that for every $(r,b)\in \Theta$ and $\xi\in \{-2,-1,+1,+2\}$,
\begin{equation}
\label{eq:no-regret}
\sum_{t = 1}^T\E_{(r',b')\sim \kappa_t}\E_{v\sim \varphi_t}[\bar h_t(r',b',v)] \ge \sum_{t = 1}^T \E_{(r',b')\sim \kappa_t}\E_{v\sim \varphi_t}[h_{r,b,\xi}(r',b',v)] - O\left(\log T + \sqrt{E_{r,b}\log T}\right),
\end{equation}
where $E_{r,b}:= \sum_{t = 1}^T\Pr_{(r',b')\sim \kappa_t}[(r',b') = (r,b)]$.

By \Cref{lm:exists-tau} and our choice of $\kappa_t$, for every $t = 1,\ldots,T$,
\[
\E_{(r',b')\sim \kappa_t}\E_{v\sim \varphi_t}[\bar h_t(r',b',v)] \le 0.
\]
By \eqref{eq:no-regret}, for every $(r,b)\in \Theta$ and $\xi\in\{-2,-1,+1,+2\}$,
\[
\sum_{t = 1}^T \E_{(r',b')\sim \kappa_t}\E_{v\sim \varphi_t}[h_{r,b,\xi}(r',b',v)] = O\left(\log T + \sqrt{E_{r,b}\log T}\right).
\]
Since $(r_t,b_t)$ is drawn from $\kappa_t$, by standard Martingale concentration (Freedman's inequality, see \Cref{lm:freedman}), with probability at least $1 - \delta^4$, for every $(r,b)\in \Theta$ and $\xi\in \{-2,-1,+1,+2\}$,
\begin{align}
\sum_{t = 1}^T \E_{v\sim \varphi_t}h_{r,b,\xi}(r_t,b_t,v) & = O\left(\log (1/\delta) + \sqrt{E_{r,b}\log (1/\delta)}\right),\label{eq:martingale-1}\\
||T_{r,b}| - E_{r,b}| & = O\left(\sqrt{E_{r,b}\log (1/\delta)}\right).\label{eq:martingale-2}
\end{align}
It remains to prove that the two inequalities above imply $\sr(p_{1:T};\ell_{1:T}) = O\left(\sqrt{(T\log T)\log(1/\delta)}\right)$.

As in the proof of \Cref{lm:truthful-reversed}, we define $\hat \varphi_{r,b}$ as the uniform mixture of $\varphi_t$ for $t\in T_{r,b}$.
This allows us to decompose the swap regret as in \eqref{eq:decomposition}. The following equation follows from the definition of $\gamma := \sqrt{(\ln T)(\ln(1/\delta))/T}$:
\begin{equation}
    \label{eq:useful}
    \ln(1/\delta) = T\gamma^2/\ln T.
\end{equation}

We say a pair $(r,b)\in \Theta$ is \emph{heavy} if $|E_{r,b}| \ge CTr^2 / \ln T$. Otherwise, we say $(r,b)$ is \emph{light}. 

For heavy pairs, by \eqref{eq:useful} we have
\[
|E_{r,b}| \ge CTr^2/\ln T \ge CT\gamma^2/\ln T = C\ln(1/\delta).
\]
Therefore, when $C$ is sufficiently large, \eqref{eq:martingale-2} implies $|T_{r,b}|\ge E_{r,b}/2$, and \eqref{eq:martingale-1} implies
\[
\sum_{t = 1}^T \E_{v\sim \varphi_t}h_{r,b,\xi}(r_t,b_t,v) = O\left(\sqrt{|T_{r,b}|\log (1/\delta)}\right) \quad \text{for every }\xi\in \{-2,-1,+1,+2\}.
\]
This means that inequalities \eqref{eq:phi-hat-1}-\eqref{eq:phi-hat-4} hold, so we have the same bound as \eqref{eq:reversed-heavy}:
\begin{equation}
\label{eq:eff-heavy}
\E_{v\sim \hat \varphi_{r,b}}[|v - b| - |v - s_{r,b}|] = O\left(\gamma\right).
\end{equation}

For light pairs, \eqref{eq:martingale-1} implies
\begin{equation}
    |T_{r,b}|\left( \frac 12 - \Pr_{v\sim \hat \varphi}[v \ge b] \right) \le \frac\gamma{2r}\cdot|T_{r,b}| + O\left(\sqrt{E_{r,b}\log(1/\delta)} \right)+ O(\log (1/\delta)). \label{eq:efficient-light}
\end{equation}
By the definition of $(r,b)$ being a light pair and \eqref{eq:useful}, we have
\begin{equation}
\label{eq:efficient-light-1}
\sqrt{E_{r,b}\ln (1/\delta)} \le \sqrt{CTr^2/\ln T}\cdot \sqrt{T\gamma^2/\ln T} = \sqrt C\cdot Tr\gamma/\ln T.
\end{equation}
By \eqref{eq:martingale-2} and the fact that $r\ge \gamma$, we have
\[
|T_{r,b}| \le E_{r,b} + O\left(\sqrt{E_{r,b}\log (1/\delta)}\right) = O\left(Tr^2/\log T\right) + O(Tr\gamma/\log T) = O\left(Tr^2/\log T\right).
\]
Therefore,
\begin{equation}
\label{eq:efficient-light-2}
\frac \gamma{2r}\cdot |T_{r,b}| = O(Tr\gamma/\log T).
\end{equation}
Plugging \eqref{eq:efficient-light-1}, \eqref{eq:efficient-light-2}, and \eqref{eq:useful} into \eqref{eq:efficient-light}, we have
\[
|T_{r,b}|\left( \frac 12 - \Pr_{v\sim \hat \varphi}[v \ge b] \right) = O(Tr\gamma/\log T) + O(T\gamma^2/\log T) = O(Tr\gamma/\log T).
\]
Similarly, we have
\[
|T_{r,b}|\left( \frac 12 - \Pr_{v\sim \hat \varphi}[v \le b] \right)= O(Tr\gamma/\log T).
\]
Therefore, by \Cref{lm:suboptimality},
\begin{equation}
\label{eq:eff-light}
|T_{r,b}|\cdot \E_{v\sim \hat \varphi_{r,b}}[|v - b| - |v - s_{r,b}|] = O(Tr\gamma/\log T).
\end{equation}
Finally, note that \eqref{eq:eff-heavy} and \eqref{eq:eff-light} correspond to \eqref{eq:reversed-heavy} and \eqref{eq:reversed-light} in the proof of \Cref{lm:truthful-reversed}, respectively. The rest of the proof is the same as the proof of \Cref{lm:truthful-reversed}.
\end{proof}

\bibliographystyle{plainnat}
\bibliography{ref}

\appendix
\section{Related Work}\label[appendix]{sec:relatedwork}

\subsection{Swap Regret Minimization}
There has been an extensive literature on swap regret minimization in the online learning literature \citep{hart2000simple, blum-mansour, hart2013simple}. A line of work studies the game-theoretic properties of swap regret minimization \citep{braverman2018selling, deng2019prior, deng2019strategizing, camara2020mechanisms, mansour2022strategizing, cai2023selling, brown2023learning, haghtalab2023calibrated}.  Recently, \citet{10.1145/3618260.3649681, peng2024fast} show that, without additional structure on the loss function and action space, swap regret admits lower bounds showing an inherent tradeoff: any algorithm must either incur regret with polynomial dependence on the number of actions (e.g., $\tilde O(\sqrt{nT})$ in the worst case) or else require $\exp(\Omega(1/\varepsilon))$ rounds exponential in the target $\varepsilon$ average swap regret. \citet{full-swap} obtain improved swap regret bound with a convex $d$-dimensional action space and structured losses, where the assumption of losses includes linearity, smoothness, strongly convex, etc. In particular, \citet{full-swap} imply an $\Tilde{O}\left(T^{2/3}\right)$ swap regret for convex and Lipschitz losses under single-dimensional action space, which we improve to the near-optimal $\Tilde{O}\left(T^{1/2}\right)$. 

\subsection{Online Calibration} There is a well-established connection between online calibration and swap regret minimization \citep{foster1997calibrated, foster1998asymptotic}. Calibrated predictions guarantee no swap regret for any best-responding decision makers. Viewed as a swap regret minimization problem, online calibration is more structured than the general setting in the sense that action space is allowed to be arbitrary, but the loss function is selected adversarially from a finite space, indexed by the random state to be predicted. On contrary, our setting considers arbitrarily adversarial loss functions with a convex and Lipschitz structure. The literature on online calibration designs algorithms with low calibration error \citep{foster1998asymptotic, abernethy2011does, luo2025simultaneous}, which mostly focuses on the $\ell_1$ calibration error \citep{sidestep, dagan2024improved}. \citet{hw24} minimizes swap regret for all downstream decision makers for binary states, which equivalently implies two loss functions for the adversary. This binary state results in an improved $\Tilde{O}(\sqrt{T})$ worst-case swap regret for downstream decision makers. Recently, \citet{peng2025high, 1050904} show that an exponential $\exp\left(\text{poly}\left(\frac{1}{\epsilon}\right)\right)$ rounds (assuming dimension $d\geq \text{poly}\left(\frac{1}{\epsilon}\right)$) is required for achieving $\epsilon$ $\ell_1$-calibration error.

\subsection{Calibration for Elicitable Properties}
\label{apdx: calibration for elicitable properties}
Calibration for elicitable properties requires a predicted statistical property is consistent with the same property of the empirical conditional distribution. \Cref{subsec: elicitable property} introduces the connections and distinctions of our swap regret minimization and existing work on calibration for elicitable properties.  \citet{jung2021moment} study the (multi-)calibration of moments and show it is impossible to calibrated variance and other higher moments. \citet{scope-multicalibration} establishes the equivalence between the possibility of (multi-)calibration and elicitability of a statistical property. \Cref{subsec: elicitable property} has introduced the connection of our work to some previous work. The line of previous work mostly focuses on the indentification calibration error in both offline and online setting, e.g., for mean \citep{hebert2018multicalibration, gopalan2022omnipredictors, gupta2022online, haghtalab2023unifying, garg2024oracle, luo2025improved, swap-multicalibration}, $q$-th quantile \citep{garg2024oracle, Rot22, swap-multicalibration}.

\paragraph{Distinctions.}
We note that prior results on calibration for elicitable properties do not apply to our swap regret of scoring functions due to the distinction in definitions. Prior work uses alternative definitions based on \emph{identification functions} of the proper scoring rule, which we refer to as \emph{identification calibration error} for elicitable properties. \citep{scope-multicalibration, swap-multicalibration}.
For any elicitable property $\Gamma$, there exists an identification function
$V:[0,1]\times Y \to \R$ such that for every distribution $q$ over $Y$ and every $\gamma\in[0,1]$, 
\[
\E_{y\sim q}\!\left[V(\gamma,y)\right] = 0
\quad\Longleftrightarrow\quad
\gamma \in \Gamma(q).
\]
Generally, the identification function $V$ is the derivative of a scoring function that elicits $\Gamma$: $V(\gamma,y)=\partial_\gamma S(\gamma,y)$ \citep{osband1985providing, lambert2011elicitation, steinwart2014elicitation}. For example, for the median, 
$V_{\mathsf{median}}(p,y)=2\I[p\ge y]-1$; for the mean, $V_{\mathsf{mean}}(p,y)=2(p-y)$; and for the $q$-quantile,
$V_{q\text{-}\mathsf{quantile}}(p,y)=q-\I[y\le p]$.

 The identification calibration error is defined as the $\ell_r$-aggregation of the identification function with results for $r\geq 2$. 
Given a sequence $(x_t,p_t,y_t)_{t=1}^T$, let
$n_p = |\{t\in[T]: p_t=p\}|$.
The $\ell_r$-calibration error in \citet{scope-multicalibration,swap-multicalibration} is
\begin{equation}
\label{eq:prior_mcal}
\mathrm{MCal}_{r}
~:=~
\sum_{p\in[0,1]}
n_p\left|
\frac{1}{n_p}\sum_{t=1}^T \I[p_t=p]\; V(p_t,y_t)
\right|^{r}.
\end{equation}
\citet{swap-multicalibration} design an algorithm that achieves $O\left(T^{\frac{1}{r+1}}\right)$ identification calibration error for several assumptions: 1) $r\geq 2$; 2) Lipschitz identification functions; 3) Lipschitz distribution of the states. Notably for 3), \citet{swap-multicalibration} and \citet{scope-multicalibration} assume the distribution of the state $Y$ is Lipschitz continuous, while we do not impose any restrictions on the distribution. 

Several distinctions arise due to different definitions. 
\begin{itemize}
    \item First, a proper scoring functions have a connection with downstream decision payoff. Given a prediction of a statistical property, the proper scoring function corresponds to the payoff of some decision maker whose decision payoff only depends on the property and best responds to the prediction. Our problem is thus equivalent to minimizing the swap regret of a downstream decision maker. 
    \item Second, the Lipschitzness of the distribution is necessary for median and quantile with identification calibration error in previous work, while our swap regret minimization do not require such an assumption, especially when the algorithm adopts a fixed binning discretization strategy of the prediction space. Intuitively, we take median calibration as an example. When $r=1$, the identification calibration error of median simplifies to the distance between the empirical quantile of the prediction $p$ and $\frac{1}{2}$: 
\begin{align*}
\mathrm{MCal}_{1}^{\mathsf{median}}
&~=~
\sum_{p\in[0,1]}
n_p\left|
2\cdot \frac{1}{n_p}\sum_{t:\,p_t=p}\I[y_t\le p]-1
\right|\\
&~=~ 2\sum_{p\in[0,1]} n_p\left|\Pr_{y\sim \tau_p}[y \le p]-\frac12\right|,
\end{align*}
i.e., for each prediction value $p$, we compute the empirical quantile level of $p$ among the associated samples $\{y_t: p_t=p\}$ and compare it to $1/2$. 

\Cref{lem: impossible identification calibration for fixed bin} shows that if an online prediction algorithm predicts from a fixed set of prediction values (which many previous algorithms do), then there exists a (non-Lipschitz) distribution where the identification calibration error is $\Theta(T)$. The main idea is that, when the true distribution of $y_t$'s is not Lipschitz and concentrates in a fixed bin, the discretization error cannot be bounded for $\mathrm{MCal}_{1}^{\mathsf{median}}$. 
\end{itemize}

\begin{lemma}
\label[lemma]{lem: impossible identification calibration for fixed bin}
    Suppose an online prediction algorithm makes predictions from a finite set $B\subseteq[0, 1]$. There exists a distribution $\tau$ such that when $\forall t, y_t\overset{\text{i.i.d.}}{\sim} \tau$, $\E[\mathrm{MCal}_{1}^{\mathsf{median}}] = \Theta(T)$. 
\end{lemma}
\begin{proof}
Let $B\subseteq[0,1]$ be finite, and write its elements in increasing order as
\[
b_1 < b_2 < \cdots < b_m.
\]
Pick any pair of adjacent values, say $b_i<b_{i+1}$. Define the outcome
distribution $\tau$ to be uniform on the open interval $(b_i,b_{i+1})$:
\[
y_t \overset{\text{i.i.d.}}{\sim} \tau := \mathrm{Unif}(b_i,b_{i+1}).
\]
(Notice $\tau$ is non-Lipschitz in the usual sense used in the identification-calibration
literature.)

Now consider \emph{any} online prediction algorithm that at each time $t$ outputs
$p_t\in B$ (possibly adaptively as a function of the past).  For each $p\in B$, let
$n_p := |\{t\in[T]: p_t=p\}|$ denote the number of times the algorithm predicts $p$.
Recall that (with the convention that the summand is $0$ when $n_p=0$)
\[
\mathrm{MCal}_{1}^{\mathsf{median}}
~=~
2\sum_{p\in B}
n_p\left|
\frac{1}{n_p}\sum_{t:\,p_t=p}\I[y_t\le p]-\frac12
\right|.
\]

We claim that for every $p\in B$, the indicator $\I[y_t\le p]$ is \emph{deterministic}
under our choice of $\tau$:
\begin{itemize}
\item If $p\le b_i$, then $y_t>b_i\ge p$ almost surely (since $y_t\in(b_i,b_{i+1})$),
so $\I[y_t\le p]=0$ a.s.
\item If $p\ge b_{i+1}$, then $y_t<b_{i+1}\le p$ almost surely, so $\I[y_t\le p]=1$ a.s.
\end{itemize}
There is no third case, because $b_i$ and $b_{i+1}$ are \emph{adjacent} in $B$, hence
no element of $B$ lies in $(b_i,b_{i+1})$.

Therefore, for each fixed $p\in B$, the empirical average inside the absolute value is
either $0$ (when $p\le b_i$) or $1$ (when $p\ge b_{i+1}$), deterministically. In either
case,
\[
\left|
\frac{1}{n_p}\sum_{t:\,p_t=p}\I[y_t\le p]-\frac12
\right|
~=~ \tfrac12,
\qquad\text{whenever } n_p>0.
\]
Plugging this into the definition gives
\[
\mathrm{MCal}_{1}^{\mathsf{median}}
~=~
2\sum_{p\in B} n_p\cdot \tfrac12
~=~
\sum_{p\in B} n_p
~=~ T,
\]
since $\sum_{p\in B} n_p = T$ always. Hence the identification calibration error is
\emph{exactly} $T$ (not merely in expectation), and in particular
\[
\mathbb{E}\bigl[\mathrm{MCal}_{1}^{\mathsf{median}}\bigr] ~=~ T ~=~ \Theta(T).
\]
This proves the lemma.
\end{proof}

\section{Martingale Concentration Bounds}
\begin{lemma}[Azuma's Inequality \citep{azuma}]
\label[lemma]{lm:azuma}
    Let $X_1,\ldots,X_n\in [-1,1]$ be a martingale difference sequence. There exists an absolute constant $C > 0$ such that for every $\delta\in (0,1/2)$, with probability at least $1-\delta$,
    \[
    \left|\sum_{i = 1}^n X_i\right| \le C\sqrt{n\log(1/\delta)}.
    \]
\end{lemma}
\begin{lemma}[Freedman’s Inequality \citep{freedman1975tail}]
\label[lemma]{lm:freedman}
    Let $X_1,\ldots,X_n\in [-1,1]$ be a martingale difference sequence. Let $V_i:= \E[X_i^2|X_1,\ldots,X_{i-1}]$ denote the conditional variance of $X_i$. There exists an absolute constant $C > 0$ such that for every $\delta\in (0,1/2)$, with probability at least $1-\delta$ we have
    \[
    \left|\sum_{i = 1}^n X_i\right| \le C\sqrt{\log(n/\delta)\sum_{i = 1}^n V_i} + C\log(n/\delta).
    \]
\end{lemma}

\section{Applying the Minimax Theorem}
\label[appendix]{sec:strategy}
We use the minimax theorem to complete the proof of the following non-constructive version of \Cref{thm:main} using \Cref{lm:truthful-reversed}:

\begin{theorem}[Non-constructive version of \Cref{thm:main}]
\label{thm:non-construct}
 For every positive integer $T \ge 2$, there exists a prediction strategy for \Cref{def:main} that guarantees
    \[
    \E[\sr(p_{1:T};\ell_{1:T})] = O(\sqrt T\log T).
    \]
\end{theorem}

To apply the minimax theorem, we need to carefully specify the strategy spaces of the predictor and the adversary. 
We note that the swap regret \eqref{eq:sr} is invariant to adding constants to the loss functions $\ell_t$, so we can assume without loss of generality that $\ell_t(0) = 0$ for every $t = 1,\ldots,T$.
Let $\Lambda$ be the family of all $1$-Lipschitz convex loss functions $\ell:[0,1]\to \R$ satisfying $\ell(0) = 0$. A deterministic predictor's strategy $P = (P_1,\ldots,P_T)$ is a sequence of functions where each $P_t:[0,1]^{t - 1}\times \Lambda^{t - 1}\to [0,1]$ decides the prediction $p_t\in [0,1]$ for the $t$-th round given the previous $t -1$ rounds' transcript $(p_1,\ldots,p_{t - 1}; \ell_1,\ldots,\ell_{t -1})\in [0,1]^{t - 1}\times \Lambda^{t - 1}$. We use $\bbP_T$ to denote the set of all such strategies $P=(P_1,\ldots,P_T)$. Similarly, an adversary's strategy is $L = (L_1,\ldots,L_T)$ with $L_t:[0,1]^{t - 1}\times \Lambda^{t - 1}\to \Lambda$ for every $t = 1,\ldots,T$. We use $\bbL_T$ to denote the set of all such $L$.

Any strategy pair $(P,L)$ uniquely determines the entire game transcript $p_1,\ldots,p_T\in [0,1]$ and $\ell_1,\ldots,\ell_T\in \Lambda$. This is done by inductively setting
\[
p_t = P_t(p_1,\ldots,p_{t - 1};\ell_1,\ldots,\ell_{t - 1})\quad \text{and} \quad \ell_t = L_t(p_1,\ldots,p_{t - 1};\ell_1,\ldots,\ell_{t - 1})
\]
for $t = 1,\ldots,T$. We define $\sr(P,L):= \sr(p_{1:T};\ell_{1:T})$. For an arbitrary set $S$, we use $\Delta_S$ to denote the set of probability distributions on $S$.

\Cref{thm:non-construct} can be equivalently stated as follows:
\begin{lemma}
\label[lemma]{lm:non-construct}
    For every positive integer $T$, there exists a distribution $\cP$ on $\bbP_T$ such that for every $L\in \bbL_T$,
    \[
    \E_{P\sim \cP}[\sr(P;L)] = O(\sqrt T \log T).
    \]
Equivalently, for every positive integer $T$,
\[
\inf\nolimits_{\cP\in \Delta_{\bbP_T}}\sup\nolimits_{L\in \bbL_T}\E_{P\sim \cP}[\sr(P;L)] = O(\sqrt T \log T).
\]
\end{lemma}

The minimax theorem requires the strategy spaces of both players to be finite. Towards making the adversary's strategy space finite, we consider a finite $\varepsilon$-net $\Lambda_\varepsilon$ of $\Lambda$ such that for every $\ell\in \Lambda$, there exists $\ell'\in \Lambda_\varepsilon$ satisfying $\|\ell' - \ell\|_\infty \le \varepsilon$. Now for every strategy $L\in \bbL$, we define $L_\varepsilon\in \bbL$ to be the strategy where whenever $L$ outputs a loss $\ell_t\in \Lambda$, the strategy $L_\varepsilon$ outputs the corresponding loss $\ell_t'\in \Lambda_\varepsilon$ with $\|\ell_t' - \ell\|_\infty \le \varepsilon$. We have $|\sr(P;L) - \sr(P;L_\varepsilon)|\le 2\varepsilon T$ for every $P\in \bbP$. Therefore, \Cref{thm:non-construct} is equivalent to the following lemma:
\begin{lemma}
\label[lemma]{lm:non-construct-finite}
    For every $\varepsilon > 0$ and every positive integer $T$,
\[
\inf\nolimits_{\cP\in \Delta_{\bbP_T}}\sup\nolimits_{L\in \bbL_T}\E_{P\sim \cP}[\sr(P;L_\varepsilon)] = O(\sqrt T \log T).
\]
\end{lemma}
We prove a stronger version of \Cref{lm:non-construct-finite} by restricting the prediction strategy $P$ to a subset $\bbP_{T,B}$ of $\bbP_T$. Here we take a finite subset $B$ of $[0,1]$ and define $\bbP_{T,B}$ as the set of all prediction strategies $P\in \bbP_T$ that are restricted to only output predictions $p_t\in B$.

\begin{lemma}
\label[lemma]{lm:non-construct-finite-2}
For every positive integer $T$,
there exists a finite subset $B\subseteq [0,1]$ such that
    for every $\varepsilon > 0$,
    \[
    \inf\nolimits_{\cP\in \Delta_{\bbP_{T,B}}}\sup\nolimits_{L\in \bbL_T}\E_{P\sim \cP}[\sr(P;L_\varepsilon)] = O(\sqrt T \log T).
    \]
\end{lemma}
We are now ready to make the strategy spaces of both players finite.
Given $\varepsilon > 0$ and a finite subset $B\subseteq [0,1]$, we define $\bbP_{T,B,\varepsilon}$ to be the set of function sequences $P = (P_1,\ldots,P_T)$ with $P_t:B^{t - 1}\times \Lambda_\varepsilon^{t - 1}\to B$ for every $t = 1,\ldots,T$. Similarly, we define $\bbL_{T,B,\varepsilon}$ to be the set of function sequences $L = (L_1,\ldots,L_T)$ with $L_t:B^{t - 1}\times \Lambda_\varepsilon^{t - 1}\to \Lambda_\varepsilon$ for every $t = 1,\ldots,T$. \Cref{lm:non-construct-finite-2} is equivalent to the following lemma:

\begin{lemma}
\label[lemma]{lm:non-construct-finite-4}
    For every positive integer $T$, there exists a finite subset $B\subseteq [0,1]$ such that for every $\varepsilon > 0$,
    \[
    \inf\nolimits_{\cP\in \Delta_{\bbP_{T,B,\varepsilon}}}\sup\nolimits_{L\in \bbL_{T,B,\varepsilon}}\E_{P\sim \cP}[\sr(P;L)] = O(\sqrt T \log T).
    \]
\end{lemma}

Note that both $\bbP_{T,B,\varepsilon}$ and $\bbL_{T,B,\varepsilon}$ are finite sets, so we can apply the minimax theorem on them:
\[
\min\nolimits_{\cP\in \Delta_{\bbP_{T,B,\varepsilon}}}\max\nolimits_{L\in \bbL_{T,B,\varepsilon}}\E_{P\sim \cP}[\sr(P;L)] = \max\nolimits_{\cL\in \Delta_{\bbL_{T,B,\varepsilon}}} \min\nolimits_{P\in \bbP_{T,B,\varepsilon}}\E_{L\sim \cL}[\sr(P;L)].
\]
Therefore, \Cref{lm:non-construct-finite-4} is equivalent to the following lemma:
\begin{lemma}
    \label[lemma]{lm:non-construct-finite-5}
    For every positive integer $T$, there exists a finite subset $B\subseteq [0,1]$ such that for every $\varepsilon > 0$,
    \[
    \max\nolimits_{\cL\in \Delta_{\bbL_{T,B,\varepsilon}}} \min\nolimits_{P\in \bbP_{T,B,\varepsilon}}\E_{L\sim \cL}[\sr(P;L)] = O(\sqrt T \log T).
    \]
\end{lemma}
\begin{proof}[Proof of \Cref{thm:non-construct} using \Cref{lm:truthful-reversed}]
By our discussion above, it suffices to prove \Cref{lm:non-construct-finite-5}. Given any mixed strategy $\cL\in \Delta_{\bbL_{T,B,\varepsilon}}$ of the adversary, \Cref{lm:truthful-reversed} gives a prediction strategy $P\in \bbP_{T,B,\varepsilon}$ that achieves $\E_{L\sim \cL}[\sr(P;L)] = O(\sqrt T\log T)$, completing the proof.
\end{proof}

\end{document}